\documentclass[11pt]{article}
\usepackage[margin=1in]{geometry}
\usepackage{amsmath,amssymb,amsthm,mathtools,booktabs,microtype}
\usepackage[numbers,sort&compress]{natbib}
\usepackage[colorlinks=true,allcolors=blue]{hyperref}
\newtheorem{definition}{Definition}\newtheorem{axiom}{Axiom}
\newtheorem{lemma}{Lemma}\newtheorem{theorem}{Theorem}
\newtheorem{proposition}{Proposition}\newtheorem{corollary}{Corollary}

\title{CPR-IE: A Compression--Prediction--Resource Intelligence Efficiency Metric}
\author{Xiantao Jiang\\
\textit{College of Information Engineering, Shanghai Maritime University}\\
\textit{Shanghai 201306, China}\\
\texttt{xtjiang@shmtu.edu.cn}}
\date{}
\begin{document}\maketitle

\begin{abstract}
Comparing intelligent systems under deployment constraints requires more than predictive accuracy. This paper develops Compression--Prediction--Resource Intelligence Efficiency (CPR-IE) as a protocol-relative ordering by representational economy, predictive quality, and resource burden. The analysis separates two questions---how raw resource consumption is represented, and how the resulting attributes are aggregated. Proportional-increment composition uniquely yields logarithmic cumulative burden, and context-independent ratio response yields power responses to compression, prediction, and burden; with reference normalization the representation is
\[
I(C,P,T)=\frac{C^{\alpha}P^{\beta}}{[1+\ln(1+T/T_0)]^{\gamma}},\qquad \alpha,\beta,\gamma>0.
\]
We prove Pareto consistency, unit invariance, boundary behavior, trade-off identities, ranking-stability regions, and cross-task aggregation. A translog parent model makes interaction restrictions explicit, and further results establish cardinal and ordinal identification, sub-Gaussian finite-sample ranking guarantees, robust selection under exponent uncertainty, and deterministic regret bounds. Minimum description length, algorithmic complexity, proper scoring rules, variational inference, and Landauer's principle motivate measurement choices but do not entail the formula. CPR-IE is a constructed efficiency representation, not a universal law or a definition of intelligence itself.
\end{abstract}

\section{Problem Formulation}
Let a system $s$ be evaluated under protocol $\Pi=(\mathcal D,\mathcal A,\mathcal B,\mathcal M,\mathcal H)$, where $\mathcal D$ is the task distribution, $\mathcal A$ the output space, $\mathcal B$ the coding convention, $\mathcal M$ the measurement boundary, and $\mathcal H$ the reporting horizon. The protocol is part of the definition. Changing it changes the measured object.

Each system produces $(C,P,T)$: compression benefit $C>0$, prediction benefit $P>0$, and resource cost $T\geq0$. CPR-IE represents an efficiency ordering on these triples. It does not replace competence, safety constraints, or a Pareto report. Its two mathematical layers must be distinguished. The map from $T$ to burden is characterized first; the aggregation of benefits and burden is characterized second.

\section{Operational Attributes}
\begin{definition}[Compression benefit]
For a fixed effective code $\mathcal B$, let $L_{\rm ref}>0$ be a reference description length and let $L_s>0$ include model, architecture, precision, decoder side information, and residual or predictive codelength. Define
\begin{equation}C_s=L_{\rm ref}/L_s.\label{eq:C}\end{equation}
The same code and precision convention must be used throughout a comparison.
\end{definition}

\begin{definition}[Prediction benefit]
For a preregistered proper predictive loss $\ell_s$ and protocol scale $\tau>0$, define
\begin{equation}P_s=\exp(-\ell_s/\tau).\label{eq:P}\end{equation}
The exponential converts equal loss improvements into equal proportional changes in $P$; it does not claim that likelihood exhausts predictive competence.
\end{definition}

\begin{definition}[Normalized resource increment]
Let $T_0>0$ have the same unit as $T$ and write $x=T/T_0$. Define
\begin{equation}x\oplus y=x+y+xy.\label{eq:oplus}\end{equation}
This operation represents successive proportional expansions because $1+x\oplus y=(1+x)(1+y)$.
\end{definition}

Energy, latency, memory, and communication are not interchangeable. Several resources should be reported separately unless an external decision model fixes their aggregation before rankings are observed.

\section{Stage I: Resource-Burden Representation}
Let $B:[0,\infty)\to[0,\infty)$ be cumulative resource burden.
\begin{axiom}[Resource order]$B$ is strictly increasing.\end{axiom}
\begin{axiom}[Proportional-increment composition]
$B(x\oplus y)=B(x)+B(y)$ for all $x,y\geq0$.
\end{axiom}
\begin{axiom}[Local calibration]$B$ is differentiable at zero and $B'(0)=1$.\end{axiom}

\begin{lemma}[Cauchy reduction]
For $G(u)=B(u-1)$, Axiom 2 is equivalent to $G(uv)=G(u)+G(v)$ on $u,v\geq1$. Moreover, $B(0)=0$ follows from Axiom 2.
\end{lemma}
\begin{proof}Set $u=1+x$ and $v=1+y$; then $uv=1+x\oplus y$. Setting $y=0$ in the composition law gives $B(x)=B(x)+B(0)$, hence $B(0)=0$.\end{proof}

\begin{theorem}[Logarithmic burden]
Axioms 1--3 hold if and only if
\begin{equation}B(x)=\ln(1+x).\label{eq:B}\end{equation}
\end{theorem}
\begin{proof}
Put $H(z)=G(e^z)$ for $z\geq0$. The lemma gives $H(z+w)=H(z)+H(w)$. Strict monotonicity of $B$ makes $H$ monotone. Every monotone additive function is linear, so $H(z)=kz$ \citep{Aczel1966}. Hence $B(x)=k\ln(1+x)$. Resource order gives $k>0$, and $B'(0)=1$ gives $k=1$. Direct substitution proves the converse.
\end{proof}

Define the multiplicative burden coordinate
\begin{equation}R(T)=1+B(T/T_0)=1+\ln(1+T/T_0).\label{eq:R}\end{equation}
The added one supplies the neutral value $R(0)=1$. The logarithm follows from the stated composition law, not from thermodynamics. Rejecting that law rejects this resource transform.

\section{Stage II: Attribute Aggregation}
Let $I:(0,\infty)^3\to(0,\infty)$ score $(C,P,R)$.
\begin{axiom}[Strong orientation]
$I$ is continuous, strictly increasing in $C,P$ and strictly decreasing in $R$.
\end{axiom}
\begin{axiom}[Context-independent ratio response]
There are positive functions $\phi_C,\phi_P,\phi_R$ such that, whenever arguments remain admissible,
\begin{align}
I(aC,P,R)/I(C,P,R)&=\phi_C(a),\label{eq:rC}\\
I(C,bP,R)/I(C,P,R)&=\phi_P(b),\label{eq:rP}\\
I(C,P,dR)/I(C,P,R)&=\phi_R(d).\label{eq:rR}
\end{align}
\end{axiom}
\begin{axiom}[Reference normalization]$I(1,1,1)=1$.\end{axiom}

Ratio independence is falsifiable. It excludes interactions such as a resource penalty whose slope depends on predictive quality. Such interactions may be appropriate in safety-critical applications.

\begin{lemma}[Power response]
Context-independent ratio response and continuity imply $\phi_j(u)=u^{a_j}$ for real constants $a_j$.
\end{lemma}
\begin{proof}
Two successive scale changes imply $\phi_j(uv)=\phi_j(u)\phi_j(v)$; concatenation consistency is therefore derived rather than assumed. Let $h_j(z)=\ln\phi_j(e^z)$. Then $h_j(z+w)=h_j(z)+h_j(w)$. Continuity of $I$ gives continuity of $\phi_j$, hence $h_j(z)=a_jz$ \citep{Aczel1966}.
\end{proof}

\begin{theorem}[Aggregation representation]
Axioms 4--6 hold if and only if
\begin{equation}I(C,P,R)=C^\alpha P^\beta R^{-\gamma},\qquad \alpha,\beta,\gamma>0.\label{eq:aggregate}\end{equation}
\end{theorem}
\begin{proof}
Starting from $(1,1,1)$ and applying Eqs.~\eqref{eq:rC}--\eqref{eq:rR} gives $I=C^{a_C}P^{a_P}R^{a_R}$. Strong orientation implies $a_C,a_P>0$ and $a_R<0$. Set $\alpha=a_C$, $\beta=a_P$, and $\gamma=-a_R$. Normalization removes the positive multiplicative constant. The converse follows by substitution.
\end{proof}

\begin{theorem}[CPR-IE characterization]
Under Axioms 1--6, the normalized cardinal representation is
\begin{equation}
\boxed{I(C,P,T)=\frac{C^\alpha P^\beta}{[1+\ln(1+T/T_0)]^\gamma}},\quad \alpha,\beta,\gamma>0.\label{eq:CPRIE}
\end{equation}
Conversely, Eq.~\eqref{eq:CPRIE} satisfies all six axioms.
\end{theorem}
\begin{proof}Combine Theorems 1 and 2 with Eq.~\eqref{eq:R}. Their converse statements prove sufficiency.\end{proof}

The uniqueness is cardinal after the ratio scale and normalization are fixed. Any strictly increasing transform preserves the same ordinal ranking but need not preserve ratio response; it therefore represents a different cardinal index.

\section{Mathematical Consequences}
\begin{proposition}[Local responses]
For $R=R(T)$,
\begin{align}
\partial\ln I/\partial\ln C&=\alpha,&
\partial\ln I/\partial\ln P&=\beta,\\
\partial\ln I/\partial T&=-\gamma/[(T_0+T)R],&
\partial\ln I/\partial\ln T&=-\gamma T/[(T_0+T)R].
\end{align}
Thus the benefit elasticities are constant, while raw-resource elasticity is state dependent.
\end{proposition}

\begin{proposition}[Boundaries]
For fixed $C,P>0$, $I(C,P,0)=C^\alpha P^\beta$ and
\[
I(C,P,T)\sim C^\alpha P^\beta/[\ln(T/T_0)]^\gamma\quad(T\to\infty).
\]
For fixed $T$, $I\to0$ when $C\downarrow0$ or $P\downarrow0$.
\end{proposition}

\begin{theorem}[Strict Pareto consistency]
If systems $a,b$ share a protocol, $C_a\geq C_b$, $P_a\geq P_b$, and $T_a\leq T_b$, with one strict inequality, then $I_a>I_b$ for every positive exponent triple.
\end{theorem}
\begin{proof}
The benefit factors for $a$ are no smaller. Since $R$ increases strictly, its inverse-power factor is no smaller. At least one factor is strictly larger.
\end{proof}
\begin{corollary}
Exponent selection changes rankings only among Pareto-incomparable systems. A strict-dominance reversal signals inconsistent preprocessing, uncertainty, or implementation error.
\end{corollary}

\begin{proposition}[Scale invariance]
Replacing $(T,T_0)$ by $(uT,uT_0)$ leaves $I$ unchanged. Replacing all $C$ by $aC$ and all $P$ by $bP$ multiplies all scores by $a^\alpha b^\beta$ and preserves rankings. Changing $T_0$ alone can change rankings and defines a new protocol.
\end{proposition}

\begin{proposition}[Ranking-stability cone]
System $a$ outranks $b$ exactly when
\begin{equation}
\alpha\ln(C_a/C_b)+\beta\ln(P_a/P_b)-\gamma\ln(R_a/R_b)>0.\label{eq:pair}
\end{equation}
For finitely many systems, every fixed ranking corresponds to an intersection of homogeneous open half-spaces in exponent space and is therefore a convex cone.
\end{proposition}
\begin{proof}Take $\ln(I_a/I_b)$. Each pairwise ordering is linear and homogeneous in $(\alpha,\beta,\gamma)$.\end{proof}

\begin{corollary}[Compensation]
At fixed $I,T$, $d\ln P/d\ln C=-\alpha/\beta$. At fixed $I,P$,
\[
d\ln C/dT=\gamma/[\alpha(T_0+T)R(T)].
\]
A hard safety or latency condition must remain a feasibility constraint rather than a compensable score component.
\end{corollary}

\begin{proposition}[Uncertainty envelope]
If the absolute errors in $\widehat{\ln C}$, $\widehat{\ln P}$, and $\widehat{\ln R}$ are bounded by $\varepsilon_C,\varepsilon_P,\varepsilon_R$, then
\begin{equation}
|\widehat{\ln I}-\ln I|\leq\alpha\varepsilon_C+\beta\varepsilon_P+\gamma\varepsilon_R.\label{eq:uncertainty}
\end{equation}
A pairwise order is certified when the estimated log-score gap exceeds both systems' error envelopes combined.
\end{proposition}
\begin{proof}Apply the triangle inequality to the affine log representation.\end{proof}

\subsection{Worked Illustration}
Set $\alpha=\beta=\gamma=1$ and $T_0=1$, and consider three systems $a,b,c$ sharing a protocol with
\[
(C_a,P_a,T_a)=(4,1,1),\qquad (C_b,P_b,T_b)=(1,4,1),\qquad (C_c,P_c,T_c)=(2,2,4).
\]
For $T=1$ the burden factor is $R=1+\ln 2\approx1.693$; for $T=4$ it is $R=1+\ln5\approx2.609$. Thus
\[
I_a=\frac{4}{1.693}\approx2.36,\qquad I_b=\frac{4}{1.693}\approx2.36,\qquad I_c=\frac{4}{2.609}\approx1.53.
\]
Systems $a$ and $b$ differ only by swapping the $C$ and $P$ coordinates; they are Pareto-incomparable and tie at equal weights. Their order is exponent-dependent: with $(\alpha,\beta)=(2,1)$, $I_a\approx9.45>2.36\approx I_b$, while with $(\alpha,\beta)=(1,2)$ the ranking reverses. System $c$ is not Pareto-dominated by $a$ or $b$, yet ranks last at equal weights because of its fourfold resource cost; with different exponents its place can change. A hard latency budget, in contrast, would exclude $c$ before any scoring.

\section{A General Interaction Representation}
Set
\begin{equation}
z=(z_1,z_2,z_3)=(\ln C,\ln P,-\ln R),
\qquad \theta=(\alpha,\beta,\gamma)>0.
\label{eq:z}
\end{equation}
Then CPR-IE is log-linear: $\ln I=\theta^\top z$. A second-order parent family is
\begin{equation}
F_G(z)=a_0+a^\top z+\frac12z^\top H z,
\qquad H=H^\top,
\label{eq:parent}
\end{equation}
with generalized index $I_G=\exp(F_G)$. The diagonal terms permit a coordinate's elasticity to vary with its own level. The off-diagonal terms encode compression--prediction, compression--resource, and prediction--resource interactions.

\begin{theorem}[CPR-IE as the independence-restricted parent model]
Within the twice differentiable family in Eq.~\eqref{eq:parent}, the following are equivalent:
\begin{enumerate}
\item the log-score increment caused by any coordinate displacement $t e_j$ is independent of both the starting point and the other coordinates;
\item $H=0$;
\item $I_G=A C^{a_1}P^{a_2}R^{-a_3}$ for some $A>0$.
\end{enumerate}
With strong orientation and reference normalization, $A=1$ and $a_1,a_2,a_3>0$, yielding CPR-IE.
\end{theorem}
\begin{proof}
For coordinate $j$,
\[
F_G(z+t e_j)-F_G(z)=t a_j+t e_j^\top Hz+\tfrac12t^2H_{jj}.
\]
Independence from the starting point $z$ for every $t$ implies $e_j^\top H=0$. Repeating for all $j$ gives $H=0$. The reverse implication is immediate. Exponentiating $a_0+a^\top z$ gives the stated power form. Orientation fixes the signs, and normalization fixes $A$.
\end{proof}

\begin{corollary}[Interaction specification test]
Within Eq.~\eqref{eq:parent}, any stable nonzero entry of $H$ is evidence against context-independent ratio response. The restricted CPR-IE model can therefore be tested against a nested interaction alternative without redefining its three measured attributes.
\end{corollary}

This theorem does not claim that the quadratic parent is universally exhaustive. It supplies the smallest smooth extension that exposes local interaction and curvature. Higher-order or nonparametric alternatives remain admissible when the data warrant them.

\section{Identification of the Exponents}
Let $z_s$ be defined by Eq.~\eqref{eq:z}. Cardinal observations satisfy
\begin{equation}
y_s=\kappa+\theta^\top z_s+u_s,
\label{eq:cardinal}
\end{equation}
where $y_s$ is an independently defined log-utility. Ordinal observations reveal only comparisons $s\succ t$, determined by $\theta^\top(z_s-z_t)>0$.

\begin{theorem}[Cardinal identification]
Let $Z$ be the matrix with rows $z_s^\top$, after removing the intercept by centering. In the noiseless model, $\theta$ is uniquely identified from cardinal outcomes if and only if $\operatorname{rank}(Z)=3$.
\end{theorem}
\begin{proof}
If $Z$ has full column rank, $Z\theta=y$ has at most one solution. If its rank is below three, there exists $v\ne0$ with $Zv=0$, so $\theta$ and $\theta+v$ generate identical outcomes whenever both satisfy the parameter restrictions.
\end{proof}

Pure ordering data cannot identify the magnitude of $\theta$: every $q\theta$, $q>0$, generates the same comparisons. Impose the simplex normalization
\begin{equation}
\Theta_1=\{\theta>0:\mathbf 1^\top\theta=1\}.
\label{eq:simplex}
\end{equation}

\begin{theorem}[Local ordinal identification]
Let $D$ contain the observed comparison differences $(z_s-z_t)^\top$, and let $Q$ be any $3\times2$ matrix whose columns span $\{v:\mathbf1^\top v=0\}$. Under a regular pairwise-response model whose likelihood depends strictly on $D\theta$, a normalized interior parameter $\theta\in\Theta_1$ is locally identified if and only if
\begin{equation}
\operatorname{rank}(DQ)=2.
\label{eq:ordinalrank}
\end{equation}
\end{theorem}
\begin{proof}
Admissible local perturbations preserving Eq.~\eqref{eq:simplex} have the form $Qv$. They are observationally null exactly when $DQv=0$. The parameter is locally unique precisely when the only such $v$ is zero, equivalently when $DQ$ has full column rank.
\end{proof}

\begin{corollary}
If all evaluated systems lie on a line in log-attribute space, or if one attribute is an exact affine combination of the others over the candidate set, ordinal data cannot separately identify all normalized exponents. More systems do not cure a deficient design geometry.
\end{corollary}

\section{Stochastic Ranking Guarantees}
Assume $n$ independent repeated measurements yield $\widehat z_s=z_s+e_s$. A centered scalar $X$ is sub-Gaussian with variance proxy $v$ when $\mathbb E e^{\lambda X}\leq e^{\lambda^2v/2}$ for every $\lambda\in\mathbb R$.

\begin{theorem}[Pairwise reversal bound]
Let the true log-score gap be $\Delta_{ab}=\theta^\top(z_a-z_b)>0$. Suppose
$\theta^\top(e_a-e_b)$ is centered sub-Gaussian with variance proxy $v_{ab}/n$. Then
\begin{equation}
\Pr(\widehat I_a\leq\widehat I_b)
\leq \exp\!\left(-\frac{n\Delta_{ab}^2}{2v_{ab}}\right).
\label{eq:reversal}
\end{equation}
\end{theorem}
\begin{proof}
A reversal requires $\theta^\top(e_a-e_b)\leq-\Delta_{ab}$. The standard Chernoff bound for a centered sub-Gaussian variable gives Eq.~\eqref{eq:reversal}.
\end{proof}

If coordinate errors are independent and their per-replicate proxies are $\sigma_{s,j}^2$, one may take
\begin{equation}
v_{ab}=\sum_{j=1}^3\theta_j^2(\sigma_{a,j}^2+\sigma_{b,j}^2).
\end{equation}

\begin{corollary}[Repeated-measurement requirement]
To make the reversal probability at most $\delta$, it is sufficient that
\begin{equation}
n\geq \frac{2v_{ab}}{\Delta_{ab}^2}\ln\frac1\delta.
\label{eq:samplesize}
\end{equation}
For $M$ prespecified pairwise claims, replacing $\delta$ by $\delta/M$ controls the familywise probability of any reversal by the union bound.
\end{corollary}

The bound becomes uninformative near a tie, as it should. It also distinguishes repeated hardware measurement from independent task replication; correlated repetitions require an effective variance or a dependence-aware concentration result.

\section{Robust Selection and Regret}
Let $\mathcal S_B$ contain systems satisfying all hard constraints, such as $T_s\leq B$ and $P_s\geq P_{\min}$. Scalar ranking is applied only after this feasibility screen. For uncertain exponents in a compact convex set $\Theta\subset\mathbb R_{++}^3$, define
\begin{equation}
s_R\in\arg\max_{s\in\mathcal S_B}\min_{\theta\in\Theta}\theta^\top z_s.
\label{eq:robust}
\end{equation}

\begin{theorem}[Finite reduction and dominance safety]
If $\Theta$ is a polytope, the inner minimum in Eq.~\eqref{eq:robust} is attained at a vertex of $\Theta$. Moreover, no strictly Pareto-dominated system can be the unique robust maximizer when its dominator is feasible.
\end{theorem}
\begin{proof}
For fixed $s$, $\theta^\top z_s$ is linear in $\theta$, so its minimum over a compact polytope occurs at a vertex. If $a$ strictly dominates $b$, then $z_a\geq z_b$ coordinatewise with one strict component. Every $\theta\in\Theta$ is strictly positive, hence $\theta^\top z_a>\theta^\top z_b$ for every $\theta$. Thus the worst-case score of $b$ cannot strictly exceed that of $a$.
\end{proof}

\begin{theorem}[Plug-in selection regret]
Let $s^*=\arg\max_{s\in\mathcal S_B}\theta^{*\top}z_s$ and
$\widehat s=\arg\max_{s\in\mathcal S_B}\widehat\theta^\top z_s$. For any norm $\|\cdot\|$ and its dual $\|\cdot\|_*$, if $\max_s\|z_s\|\leq M$, then
\begin{equation}
0\leq \theta^{*\top}z_{s^*}-\theta^{*\top}z_{\widehat s}
\leq 2M\|\widehat\theta-\theta^*\|_*.
\label{eq:regret}
\end{equation}
\end{theorem}
\begin{proof}
Add and subtract $\widehat\theta^\top z_{s^*}$ and $\widehat\theta^\top z_{\widehat s}$. The middle difference is nonpositive by optimality of $\widehat s$. H\"older's inequality bounds each remaining term by $M\|\widehat\theta-\theta^*\|_*$.
\end{proof}

\begin{corollary}[Joint measurement and parameter error]
If $|\widehat\theta^\top\widehat z_s-\theta^{*\top}z_s|\leq\varepsilon$ uniformly over feasible systems, plug-in selection has log-utility regret at most $2\varepsilon$.
\end{corollary}
\begin{proof}The usual two-term empirical-maximization decomposition bounds regret by two uniform deviations.\end{proof}

\section{Cross-Task Aggregation}
For externally fixed $w_j\geq0$ with $\sum_jw_j=1$, define $I_{\rm geo}=\prod_jI_j^{w_j}$. Then
\[
\ln I_{\rm geo}=\alpha\sum_jw_j\ln C_j+\beta\sum_jw_j\ln P_j-\gamma\sum_jw_j\ln R_j.
\]
This rule is invariant to uniform replication of tasks and preserves common multiplicative improvements. Its weights remain normative. Arithmetic averaging encodes a different compensation rule.

\section{Formal Comparison with Alternative Separable Representations}
Write the beneficial resource attribute as $Z=R^{-1}$. CPR-IE is the weighted geometric representation $C^\alpha P^\beta Z^\gamma$, generalizing the QuIDE quantized-intelligence index \citep{jiang2026quide}. It places CPR-IE within multiattribute value theory, but its invariance assumptions differ from those of additive independence \citep{KeeneyRaiffa1976}.

The additive form
\begin{equation}
U_A=w_Cu_C(C)+w_Pu_P(P)+w_Zu_Z(Z)
\label{eq:additive}
\end{equation}
is appropriate when preferential independence supports additive separability. It is Pareto-monotone for positive weights and increasing component values, but generally violates CPR-IE ratio independence: the ratio $U_A(aC,P,Z)/U_A(C,P,Z)$ depends on $P$ and $Z$. Additive scores are also sensitive to arbitrary affine rescaling of their components unless weights are recalibrated.

The constant-elasticity-of-substitution family
\begin{equation}
U_{\rho}=\left(w_CC^\rho+w_PP^\rho+w_ZZ^\rho\right)^{1/\rho},
\qquad \rho\ne0,
\label{eq:ces}
\end{equation}
permits a tunable substitution elasticity. With positive weights summing to one,
\begin{equation}
\lim_{\rho\to0}U_\rho=C^{w_C}P^{w_P}Z^{w_Z}.
\end{equation}
Thus normalized CPR-IE is the Cobb--Douglas limit of CES when its exponents are interpreted as weights summing to one. Unnormalized exponents additionally control returns to joint scale. For $\rho\ne0$, CES rejects coordinatewise ratio independence because the response to scaling one coordinate depends on the other coordinates.

A translog representation,
\begin{equation}
\ln U_T=a_0+\sum_j a_j\ln x_j+\frac12\sum_j\sum_k b_{jk}\ln x_j\ln x_k,
\qquad x=(C,P,Z),
\label{eq:translog}
\end{equation}
nests the CPR-IE log form when every $b_{jk}=0$. Nonzero interaction coefficients allow the compression--\allowbreak prediction trade-off or resource penalty to vary by operating point. Equation~\eqref{eq:translog} is therefore a direct specification test: stable nonzero interactions are evidence against Axiom 5, not a refinement already implied by CPR-IE.

Noncompensatory rules such as $U_{\min}=\min\{\widetilde C,\widetilde P,\widetilde Z\}$ and constrained Pareto selection reject smooth substitution altogether. They are preferable when failure in one dimension cannot be offset by gains elsewhere. CPR-IE should consequently be applied after hard feasibility and safety constraints, not in place of them.

These representations cannot be ranked in the abstract. Additive utility assumes additive independence; CES assumes a particular global substitution structure; translog allows local interactions; minimum and Pareto rules limit compensation. CPR-IE is distinguished by multiplicative scale response and exact log-linearity. The empirical question is whether those restrictions are stable enough to improve out-of-sample decisions.

\section{Provenance and Non-Implication}
MDL formalizes model selection by description length \citep{Rissanen1978,Grunwald2007}. Kolmogorov complexity supplies an ideal shortest-program quantity, but is uncomputable and machine dependent up to an additive constant \citep{LiVitanyi2008}. These theories discipline the construction of $C$; they do not prescribe $C^\alpha$ or the joint index.

Proper scoring rules support principled evaluation of probabilistic predictions \citep{GneitingRaftery2007}. Variational inference relates the evidence lower bound to KL divergence under a specified model \citep{Blei2017}; the free-energy principle develops a broader generative account of perception and action \citep{Friston2010}. These results motivate choices for $P$ but do not imply Eq.~\eqref{eq:CPRIE}. Variational free energy is not thermodynamic free energy.

Landauer's analysis links logically irreversible operations to a minimum heat cost in an idealized setting \citep{Landauer1961}; reversible computing sharpens the distinction between logical and physical irreversibility \citep{Bennett1982}. This establishes that computation is physical, but neither equates accelerator energy with the Landauer limit nor implies logarithmic burden. Equation~\eqref{eq:B} follows from Axioms 1--3.

Universal intelligence integrates goal achievement across environments using algorithmic weights \citep{LeggHutter2007}. CPR-IE addresses a narrower question: efficient realization under a declared deployment protocol. It must not be presented as a universal intelligence measure.

\section{Logical Nonredundancy and Countermodels}
The original eight-axiom formulation contained two redundancies. Resource identity follows from composition by setting one increment to zero. Multiplicative consistency of each $\phi_j$ follows by applying the ratio-response equation twice. The six-axiom system above removes both assumptions.

The remaining restrictions have distinct roles. The following countermodels show what becomes unidentified when each structural requirement is removed. They are logical witnesses, not proposed empirical alternatives.

\begin{proposition}[Resource-side nonredundancy]
The composition and calibration restrictions cannot be deleted without enlarging the admissible class.
\end{proposition}
\begin{proof}
If proportional-increment composition is deleted, $B(x)=x$ is strictly increasing and satisfies $B'(0)=1$, but
$B(x\oplus y)=x+y+xy\ne x+y=B(x)+B(y)$ for $xy>0$. Hence the logarithm is no longer characterized. If local calibration is deleted, every $B_k(x)=k\ln(1+x)$ with $k>0$ satisfies resource order and composition, so the burden scale is unidentified. Resource order excludes $B_k$ with $k\leq0$ and, more generally, supplies the regularity that eliminates nonmonotone additive solutions after the logarithmic change of variables.
\end{proof}

\begin{proposition}[Aggregation-side nonredundancy]
Orientation, context-independent ratio response, and normalization impose separate restrictions.
\end{proposition}
\begin{proof}
If orientation is deleted, $I=C^{-1}P/R$ satisfies ratio response and normalization but rewards worse compression. If ratio independence is deleted, $I=(C+P)/(2R)$ is continuous, normalized, and correctly oriented, yet
\[
\frac{I(aC,P,R)}{I(C,P,R)}=\frac{aC+P}{C+P}
\]
depends on the background level $P$. It is therefore not CPR-IE. If normalization is deleted, $A C^\alpha P^\beta R^{-\gamma}$ is admissible for every $A>0$, so the score level is unidentified although rankings remain unchanged.
\end{proof}

The witnesses establish functional nonredundancy of composition, calibration, ratio independence, orientation, and normalization. They also clarify a limitation of the word ``independence'': the regularity and sign clauses grouped within resource order and strong orientation are composite assumptions. A fully atomized logical-independence analysis would split continuity, monotonicity, and sign orientation into separate axioms; doing so changes notation but not the representation.

\section{Falsifiability}

The proposal can fail at three levels. Measurement failure occurs when $C,P,T$ cannot be estimated reliably. Structural failure occurs when operational preferences show systematic interactions or reject proportional resource composition. Decision failure occurs when CPR-IE-guided choices underperform simpler metrics out of sample. Citation to established theory cannot repair any of these failures.

The axiomatic route can look circular: proportional-increment composition encodes logarithmic burden almost by definition, and ratio independence encodes the power form. The point of the axioms is not to derive a surprising conclusion from surprising premises, but to make the premises explicit and individually testable. Axiom 2 predicts that a given proportional increment adds a burden independent of the current load---the increment that doubles $1+x$ ($y=1$) adds exactly $\ln 2$ at every base level $x$---which data reject if measured burden instead grows superlinearly. Axiom 5 predicts that the proportional gain from doubling $C$ is the same whether $P$ is large or small; a stable interaction, in which compression matters more at low predictive quality, surfaces as a nonzero off-diagonal of the translog parent model and is direct evidence against the axiom. The characterization therefore does not assume these invariances for free; it converts them from implicit habits into explicit, refutable claims, and the alternative representations compared earlier are what one falls back to when they fail.

\section{Conclusion}
Separating resource representation from attribute aggregation gives CPR-IE a noncircular foundation. Proportional resource composition produces logarithmic burden; context-independent ratio response produces the multiplicative power form. The interaction parent model identifies precisely what the restriction excludes. Rank conditions separate cardinal from ordinal identification, concentration bounds quantify measurement-induced reversals, and robust optimization converts exponent uncertainty into a finite decision rule with a regret guarantee. These results strengthen CPR-IE as a theory of protocol-relative efficiency while leaving its practical value open to empirical rejection.

\bibliographystyle{unsrtnat}\bibliography{references}

@article{Rissanen1978,
  title     = {Modeling by shortest data description},
  author    = {Rissanen, Jorma},
  journal   = {Automatica},
  volume    = {14},
  number    = {5},
  pages     = {465--471},
  year      = {1978},
  publisher = {Elsevier}
}

@book{Grunwald2007,
  title     = {The minimum description length principle},
  author    = {Gr{\"u}nwald, Peter D},
  year      = {2007},
  publisher = {The MIT press}
}

@book{LiVitanyi2008,
  author    = {Li, Ming and Vit\'anyi, Paul},
  title     = {An Introduction to Kolmogorov Complexity and Its Applications},
  edition   = {3},
  publisher = {Springer},
  year      = {2008},
  doi       = {10.1007/978-0-387-49820-1}
}

@article{GneitingRaftery2007,
  title     = {Strictly proper scoring rules, prediction, and estimation},
  author    = {Gneiting, Tilmann and Raftery, Adrian E},
  journal   = {Journal of the American statistical Association},
  volume    = {102},
  number    = {477},
  pages     = {359--378},
  year      = {2007},
  publisher = {Taylor \& Francis}
}

@article{Blei2017,
  title     = {Variational inference: A review for statisticians},
  author    = {Blei, David M and Kucukelbir, Alp and McAuliffe, Jon D},
  journal   = {Journal of the American statistical Association},
  volume    = {112},
  number    = {518},
  pages     = {859--877},
  year      = {2017},
  publisher = {Taylor \& Francis}
}

@article{Friston2010,
  title     = {The free-energy principle: a unified brain theory?},
  author    = {Friston, Karl},
  journal   = {Nature reviews neuroscience},
  volume    = {11},
  number    = {2},
  pages     = {127--138},
  year      = {2010},
  publisher = {Nature Publishing Group UK London}
}

@article{Landauer1961,
  title     = {Irreversibility and heat generation in the computing process},
  author    = {Landauer, Rolf},
  journal   = {IBM journal of research and development},
  volume    = {5},
  number    = {3},
  pages     = {183--191},
  year      = {1961},
  publisher = {Ibm}
}

@article{Bennett1982,
  title     = {The thermodynamics of computation—a review},
  author    = {Bennett, Charles H},
  journal   = {International Journal of Theoretical Physics},
  volume    = {21},
  number    = {12},
  pages     = {905--940},
  year      = {1982},
  publisher = {Springer}
}

@article{LeggHutter2007,
  title     = {Universal intelligence: A definition of machine intelligence},
  author    = {Legg, Shane and Hutter, Marcus},
  journal   = {Minds and machines},
  volume    = {17},
  number    = {4},
  pages     = {391--444},
  year      = {2007},
  publisher = {Springer}
}

@book{Aczel1966,
  title     = {Lectures on Functional Equations and Their Applications},
  author    = {Acz{\'e}l, J{\'a}nos},
  year      = {1966},
  publisher = {Academic Press},
  isbn      = {9780120437504}
}

@book{KeeneyRaiffa1976,
  title     = {Decisions with Multiple Objectives: Preferences and Value Tradeoffs},
  author    = {Keeney, Ralph L. and Raiffa, Howard},
  year      = {1976},
  publisher = {John Wiley \& Sons},
  isbn      = {9780471465102}
}

@article{jiang2026quide,
  title   = {{QuIDE}: Mastering the Quantized Intelligence Trade-off via Active Optimization},
  author  = {Jiang, Xiantao},
  journal = {arXiv preprint arXiv:2605.10959},
  year    = {2026}
}
\end{document}